%% file: main.tex
\documentclass[onecolumn]{IEEEtran}
\IEEEoverridecommandlockouts
\usepackage{hyperref}
\usepackage{cite}
\usepackage{amsmath,amssymb,amsfonts}
\usepackage{algorithmic}
\usepackage{graphicx}
\usepackage{textcomp}
\usepackage{xcolor}
\usepackage{blindtext}
\usepackage{amsthm}
\newtheorem{definition}{Definition}
\newtheorem{proposition}{Proposition}
\usepackage[margin=1in]{geometry}
\usepackage{array}
\usepackage{subcaption}
\usepackage{multirow}
\usepackage{booktabs}
\usepackage{pdfpages}
\usepackage{fancyhdr}
\def\BibTeX{{\rm B\kern-.05em{\sc i\kern-.025em b}\kern-.08em
    T\kern-.1667em\lower.7ex\hbox{E}\kern-.125emX}}
\begin{document}

\title{Progressive Semantic Communication for Efficient Edge-Cloud Vision-Language Models}

\author{
Cyril Shih-Huan Hsu$^{*}\qquad$\thanks{
$^{*}$Informatics Institute, University of Amsterdam, The Netherlands.
Email: \texttt{\{s.h.hsu, c.papagianni\}@uva.nl}
}
\and
Wig Yuan-Cheng Cheng$^{\dagger}\qquad$\thanks{
$^{\dagger}$Open-EP Community,
Taoyuan City 330049, Taiwan (R.O.C.).
Email: \texttt{onlywig@gmail.com}
}
\and
Chrysa Papagianni$^{*}$
}

\maketitle
\thispagestyle{fancy}

\begin{abstract}
Deploying Vision–Language Models (VLMs) on edge devices remains challenging due to their substantial computational and memory demands, which exceed the capabilities of resource-constrained embedded platforms. Conversely, fully offloading inference to the cloud is often impractical in bandwidth-limited environments, where transmitting raw visual data introduces substantial latency overhead. While recent edge–cloud collaborative architectures attempt to partition VLM workloads across devices, they typically rely on transmitting fixed-size representations, lacking adaptability to dynamic network conditions and failing to fully exploit semantic redundancy.
In this paper, we propose a progressive semantic communication framework for edge–cloud VLM inference, using a Meta AutoEncoder that compresses visual tokens into adaptive, progressively refinable representations, enabling plug-and-play deployment with off-the-shelf VLMs without additional fine-tuning. This design allows flexible transmission at different information levels, providing a controllable trade-off between communication cost and semantic fidelity.
We implement a full end-to-end edge–cloud system comprising an embedded NXP i.MX95 platform and a GPU server, communicating over bandwidth-constrained networks.
Experimental results show that, at 1 Mbps uplink, the proposed progressive scheme significantly reduces network latency compared to full-edge and full-cloud solutions, while maintaining high semantic consistency even under high compression. The implementation code will be released upon publication at \url{https://github.com/open-ep/ProSemComVLM}.
\end{abstract}
\begin{figure}[h]
    \centering
\includegraphics[width=0.9\columnwidth]{Figures/concept.png}
    \caption{\textbf{Progressively refinable latent representation via Meta AutoEncoder.} The fixed-length latent is structured as an ordered token sequence, where prefixes provide increasingly accurate approximations of the full representation. This enables monotonic improvement in reconstruction quality as transmission progresses, supporting efficient incremental communication and precise control over the trade-off between bandwidth usage and semantic accuracy.}
    \label{fig:concept}
\end{figure}


\section{Introduction}
Vision-language models (VLMs) have rapidly emerged as a cornerstone of multimodal artificial intelligence, enabling machines to jointly reason over visual and textual inputs for tasks such as image captioning, visual question answering, and multimodal interaction. Despite their strong performance, the deployment of VLMs in real-world systems remains challenging due to their substantial computational and memory requirements. Large-scale models often exceed the capabilities of resource-constrained edge devices, while purely cloud-based solutions introduce significant latency and communication overhead, limiting their applicability in real-time and bandwidth-constrained scenarios.

To address these challenges, recent research has explored collaborative edge–cloud (or edge-to-server) inference paradigms, which leverage the modular structure of VLMs to distribute computation across devices. In such frameworks, lightweight processing, such as visual encoding or preliminary inference, is performed at the edge, while more computationally intensive language reasoning is delegated to centralized servers. Beyond simple partitioning, more advanced approaches further optimize communication efficiency by transmitting only task-relevant information or adaptively requesting additional data based on model uncertainty, thereby balancing accuracy, latency, and bandwidth usage. These developments highlight a growing trend toward semantic-aware and resource-efficient VLM deployment, where the goal is not only to distribute computation, but also to minimize unnecessary data transmission while preserving task performance.

However, existing edge–cloud VLM systems still face several key limitations. First, most approaches rely on transmitting either raw pixels or fixed-size latent representations, which are inherently non-adaptive and fail to account for dynamic network conditions or task-dependent semantic requirements (e.g.,~\cite{Li2025DistributedVE_ec3}). 
Second, while some methods incorporate semantic-aware or uncertainty-driven transmission strategies (e.g.,~\cite{song2025collaborative_ec4, 10947590}), they typically adopt non-progressive mechanisms, where additional information is transmitted in coarse steps without reusing previously transmitted content. This results in inefficient communication under iterative transmission and lacks a principled way to achieve fine-grained control over the trade-off between communication cost and semantic fidelity.
These limitations highlight the need for a more flexible framework that enables adaptive, progressive, and semantically meaningful transmission for edge–cloud VLM inference.

In this paper, we propose a progressive semantic communication framework for efficient edge–cloud VLM deployment, with the following contributions:
\begin{itemize}
\item[1.] We introduce a \textit{progressively refinable latent representation} via a plug-and-play Meta AutoEncoder (MetaAE), which enables incremental transmission and reuse of partial visual semantics, thereby significantly improving flexibility and communication efficiency compared to conventional non-progressive schemes.
\item[2.] We design a \textit{semantic quality-aware transmission control} that adaptively determines whether additional information should be transmitted, enabling a controllable trade-off between communication cost and task performance.
\item[3.] We implement a \textit{real-world end-to-end edge–cloud testbed} on an embedded NXP i.MX95 platform and a GPU server, demonstrating the practical feasibility of the proposed framework under bandwidth-constrained conditions, and validating its effectiveness in reducing latency while preserving semantic consistency.
\end{itemize}
The rest of this paper is organized as follows. Section~\ref{sec:related} provides context through related work. Section~\ref{sec:method} introduces the proposed framework. We then detail our experimental setup in Section~\ref{sec:exp} and evaluate the framework in Section~\ref{sec:results}. We conclude this work in Section~\ref{sec:conclude}.

\section{Related Work}\label{sec:related}
\subsection{Vision Language Models}
VLMs have progressed toward general-purpose multimodal reasoning by coupling pretrained vision encoders with large language models and improving cross-modal alignment and instruction following. Early modular designs such as Flamingo~\cite{alayrac2022flamingo} show that bridging powerful vision-only and language-only backbones enables flexible few-shot prompting, while BLIP-2~\cite{pmlr-v202-li23q} highlights a complementary efficiency point: keeping large backbones frozen and training a lightweight connector can yield strong VL capabilities with fewer trainable parameters. 
Instruction-tuned models such as LLaVA~\cite{NEURIPS2023_LLAVA} further operationalize VLMs as conversational assistants by training on instruction-following multimodal data, shifting research emphasis from narrow benchmarks to interactive usage patterns that often demand low latency and robust deployment. 
More recent studies (e.g., VILA~\cite{lin2024vila}) systematically analyze pretraining and data mixture choices that materially affect both performance and engineering cost, while large technical reports such as Qwen3-VL~\cite{bai2025qwen3} reflect scaling trends that intensify resource pressures and motivate deployment optimization.

\subsection{Semantic Communication}
Semantic communication research reframes transmission objectives from reconstructing raw signals to preserving task-relevant information, increasingly using learned representations aligned to downstream goals. Transformer-based semantic image transmission (WITT~\cite{yang2023witt_sc1}) demonstrates end-to-end learned coding optimized for channel conditions, providing a base layer that can be reused when vision transformers act as the ``front end" of VLMs. 
Explicit vision-language semantic communication is explored in VLM-CSC~\cite{jiang2025visual_sc2}, which transmits high-density textual semantics extracted from images, and incorporates memory and noise-adaptive mechanisms, thereby bridging cross-modal semantics and channel robustness. 
Token-based paradigms (TokCom~\cite{qiao2025token_sc3}) further align with modern VLM/LLM interfaces by treating tokens as the unit of communication and exploiting cross-modal context during reconstruction or generation. VLF-MSC~\cite{ahn2025vlf_sc4} proposes sending a compact shared vision-language feature that conditions both language generation and image generation at the receiver. 
Video TokenCom~\cite{men2026video_sc5} extends this direction by allocating rate/protection based on textual intent and token importance, which suggests principled ways to couple application goals with source–channel adaptation. 

\subsection{Edge-Cloud Deployment}
Edge–cloud deployment work addresses the systems bottlenecks that arise when VLMs must run under tight latency, energy, and bandwidth constraints. Empirical evaluation on smartphones shows that even mobile-oriented VLM stacks can remain CPU-bound and underutilize accelerators, motivating careful runtime and partitioning design~\cite{guerrero2025efficient_ec2}. 
At the architectural level, Distributed VLMs~\cite{Li2025DistributedVE_ec3} exploit the modularity of VLMs by offloading visual encoding to edge devices while retaining language generation on servers, improving throughput via parallelization.
Collaborative strategies can further reduce uplink overhead by transmitting only task-relevant information. For example,~\cite{song2025collaborative_ec4} uses VLM attention for patch selection and an entropy-based uncertainty trigger for selective retransmission. Similarly,~\cite{10947590} introduces a region-of-interest generation mechanism to reduce bandwidth consumption during transmission. These approaches align with the principle of semantic-aware communication in edge–cloud inference loops.

\begin{figure*}[t]
     \centering
     \includegraphics[width=\textwidth]{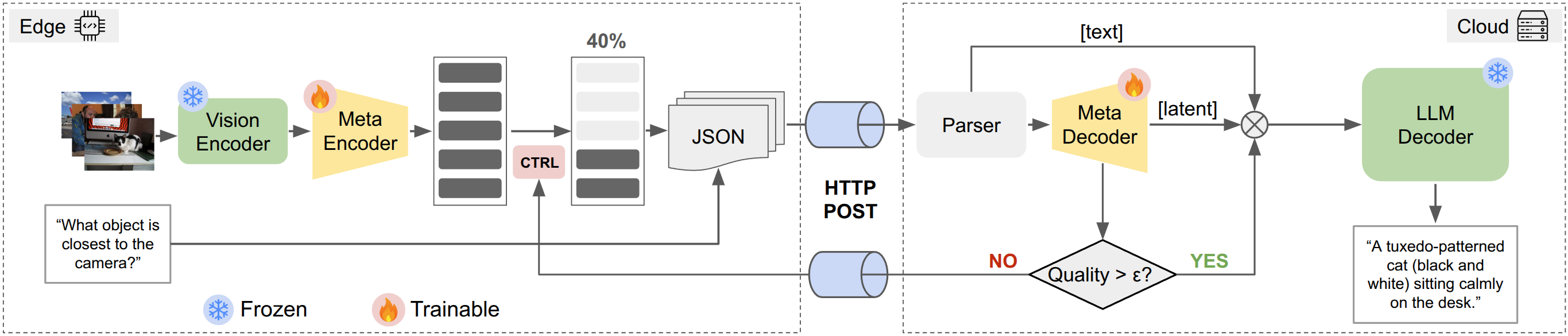}
    \caption{\textbf{System overview of the proposed progressive edge–cloud semantic communication framework.} The edge encodes visual inputs into progressively refinable representations and transmits partial latent segments under bandwidth constraints. The cloud reconstructs the latent via a meta decoder and employs semantic quality-aware transmission control to adaptively request additional information until a target quality threshold is met, after which the LLM decoder generates the final output.}
    \label{fig:system}
\end{figure*}
\section{Methodology}\label{sec:method}
\subsection{System Overview}
We propose an edge–cloud collaborative framework for progressive semantic communication with VLMs, as illustrated in Fig.~\ref{fig:system}. The system is designed to minimize communication overhead while preserving task-relevant semantic information through progressive and adaptive transmission over a practical RESTful interface.
At the \textit{edge side}, an input image is first processed by a lightweight vision encoder, producing high-dimensional latent representations. These latents are further compressed by a meta encoder into a progressively refinable representation.
This representation is explicitly ordered, such that earlier segments capture coarse semantic information, while subsequent segments provide incremental refinements that progressively improve the reconstruction quality.
A transmission control module (CTRL in Fig.~\ref{fig:system}) selects an initial subset (e.g., $40$\%) of the representation based on communication constraints. The selected segments are serialized into a compact JSON payload and transmitted to the cloud via an HTTP POST request.
At the \textit{cloud side}, the received payload is parsed and passed through a meta decoder to reconstruct the latent representation. A semantic quality estimator then evaluates whether the reconstructed latent satisfies a predefined quality threshold $\epsilon$ for the downstream task. If the estimated quality is sufficient, the reconstructed latent is combined with the textual input and forwarded to the LLM decoder to generate the final response.
Otherwise, the system enters a closed-loop refinement process: the cloud sends a feedback signal via an HTTP response requesting additional representation segments from the edge. The edge then transmits further incremental information, which is accumulated at the cloud to progressively improve the reconstruction. This process repeats until the target semantic quality is achieved (see Sec.~\ref{sec:metaAE}).
\subsection{Progressively Refinable Representations}\label{sec:prp}
Let $\mathbf{z} \in \mathbb{R}^{d}$ denote the visual latent produced by the vision encoder for an input image $\mathbf{x}$. The proposed meta encoder $f_{\mathrm{enc}}(\cdot)$ transforms $\mathbf{z}$ into a progressively refinable representation
\begin{equation}\label{eq:enc}
\mathbf{r} = f_{\mathrm{enc}}(\mathbf{z}) \in \mathbb{R}^{d_r}.
\end{equation}

To enable progressive transmission, we partition $\mathbf{r}$ into $K$ chunks:
\begin{equation}\label{eq:latent}
\mathbf{r} = \big[r_1, r_2, \dots, r_K\big],\quad \sum_{k=1}^{K} d_k = d_r.
\end{equation}
where $r_k \in \mathbb{R}^{d_k}$ denotes the $k$-th latent chunk.
Unlike conventional compressed representations, the proposed representation is \emph{explicitly ordered}. Specifically, the chunk index reflects transmission priority: earlier chunks capture coarse semantic information, while later chunks provide progressively finer refinements. This ordered structure enables coarse-to-fine reconstruction from partial transmissions.

\begin{definition}(Latent Transmission Level, \textbf{LTL}).
For a total of $K$ ordered latent chunks, the \emph{Latent Transmission Level} $\ell \in \{1,2,\dots,K\}$ denotes the transmission state in which the first $\ell$ chunks are available at the receiver. The corresponding partial representation is defined as $\mathbf{r}(\ell) = \big[r_1, r_2, \dots, r_{\ell}\big]$.
\end{definition}
Given the cumulative representation $\mathbf{r}(\ell)$, the meta decoder predicts a level-dependent approximation of the original full latent $\mathbf{z}$:
\begin{equation}\label{eq:recon}
\hat{\mathbf{z}}(\ell) = f_{\mathrm{dec}}\!\left(\mathbf{r}(\ell)\right), \quad \hat{\mathbf{z}}(\ell) \in \mathbb{R}^{d}.
\end{equation}
As $\ell$ increases, $\hat{\mathbf{z}}(\ell)$ is expected to approach $\mathbf{z}$ with progressively higher fidelity. We refer to the normalized LTL $\ell/K$ as the \textit{LTL ratio} hereafter for simplicity.

The above formulation enables a progressive transmission scheme, where information is delivered in an incremental and reusable manner. In contrast, conventional non-progressive schemes transmit representations corresponding to each quality level independently, without reusing previously transmitted information. We next show that the proposed progressive scheme is strictly more communication-efficient under iterative transmission.

\begin{proposition}
Assume a representation is divided into $K$ equal-sized chunks, each with transmission cost $b$. To reach transmission level $K$ through sequential refinement, a progressive scheme incurs strictly lower total communication cost than a non-progressive scheme that retransmits the full representation at each level.
\end{proposition}

\begin{proof}
In the progressive scheme, only the incremental chunk $\Delta \mathbf{r}(\ell) = r_{\ell}$ is transmitted at each level $\ell$. Since each chunk has size $b$, the total transmission cost to reach level $K$ is
\begin{equation}\label{eq:cost_pg}
\sum_{\ell=1}^{K} b = Kb.
\end{equation}

In contrast, in a non-progressive scheme, the representation corresponding to level $\ell$ must be transmitted in full at each step, with cost $\ell b$. Therefore, the total transmission cost is
\begin{equation}\label{eq:cost_npg}
\sum_{\ell=1}^{K} \ell b = b \sum_{\ell=1}^{K} \ell = \frac{K(K+1)}{2}b.
\end{equation}

Comparing Eq.~\eqref{eq:cost_pg} and Eq.~\eqref{eq:cost_npg}, we have
\begin{equation}
\frac{K(K+1)}{2}b > Kb, \quad \forall K > 1,
\end{equation}
which shows that the non-progressive scheme incurs strictly higher communication cost. 
Moreover, the total transmission cost of the progressive scheme scales linearly with the number of levels, i.e., $\mathcal{O}(K)$, whereas the non-progressive scheme scales quadratically, i.e., $\mathcal{O}(K^2)$.
\end{proof}

\subsection{Meta AutoEncoder}\label{sec:metaAE}
To realize the progressively refinable representation, we design a \emph{Meta AutoEncoder} with three components: ($i$) an encoder that maps the original latent into ordered transmission chunks, ($ii$) a prefix masking module that simulates partial transmission, and ($iii$) a decoder that reconstructs the latent while estimating its quality. We implement both encoder and decoder using Transformer encoder layers~\cite{vaswani2017attention}. As a result, the representation is naturally structured as a sequence of tokens, where each token corresponds to a transmission unit. In the following, we use the term \emph{token} instead of chunk for clarity and consistency with transformer architectures.

\vspace{1mm}
\noindent \textbf{Meta Encoder.} Given the visual latent $\mathbf{z} \in \mathbb{R}^{d}$, the encoder produces an ordered representation as defined in Eq.~\eqref{eq:enc} and \eqref{eq:latent},
where each token is arranged according to transmission priority. Earlier tokens capture coarse semantic information, while later token provide finer refinements.

\vspace{1mm}
\noindent \textbf{Prefix Masking Module.}
To simulate progressive transmission during training, a transmission level $\ell \in \{1, \dots, K\}$ is sampled uniformly at random. Only the prefix $\mathbf{r}(\ell) = \big[r_1, r_2, \dots, r_{\ell}\big]$ is kept, while the remaining tokens are replaced with \textit{zero padding tokens}. 
This strategy is related to nested dropout~\cite{rippel2014learning}, which learns ordered representations by stochastically truncating latent dimensions. In contrast, we apply prefix masking over transmission tokens to support progressively refinable communication.
This results in a masked representation $\tilde{\mathbf{r}}(\ell)$, which serves as input to the decoder. While uniform sampling provides a simple and unbiased training strategy, exploring optimal sampling distributions for different deployment scenarios is left for future work.
The theoretical justification for how prefix masking encourages progressively refinable representations is provided in Appendix~\ref{app:progressive-proof}.
Fig.~\ref{fig:mask} illustrates the masking process with a PyTorch implementation.
\begin{figure}[h]
    \centering
\includegraphics[width=0.6\columnwidth]{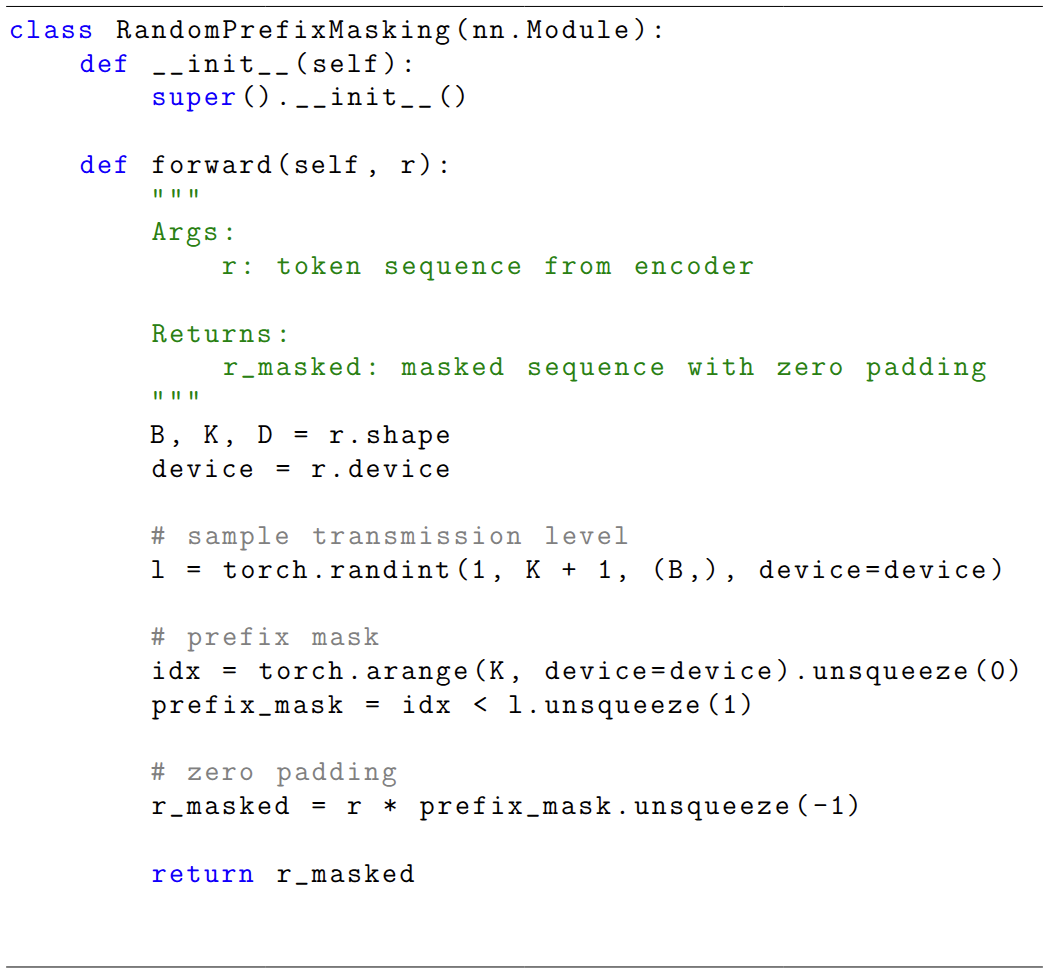}
    \caption{Example code of prefix masking module in PyTorch.}
    \label{fig:mask}
\end{figure}

\vspace{1mm}
\noindent \textbf{Meta Decoder with Semantic Quality Estimation.}
The decoder $f_{\mathrm{dec}}(\cdot)$ is implemented as a stack of Transformer encoder layers. In addition to the masked token sequence, a learnable \emph{error token} is prepended to the input sequence~\cite{devlin2019bert}. The decoder output is used for two tasks:
($i$) the token-wise outputs corresponding to the $K$ token positions are passed through a reconstruction head to predict the full latent, as described in Eq.~\eqref{eq:recon}.
Notably, the decoder always reconstructs the full latent $\mathbf{z}$, regardless of the transmission level $\ell$ (see Eq.~\eqref{eq:loss_z});
($ii$) the hidden state associated with the error token is fed into an auxiliary prediction head to estimate the reconstruction error:
\begin{equation}\label{eq:err_pred}
\hat{e}(\ell) = g_{\mathrm{err}}\!\left([f_{\mathrm{dec}}(\tilde{\mathbf{r}}(\ell))]_{\text{err}} \right),
\end{equation}
where $g_{\mathrm{err}}(\cdot)$ denotes the error prediction head, and $[\cdot]_{\text{err}}$ extracts the hidden state associated with the prepended error token from the decoder output sequence.

\vspace{1mm}
\noindent \textbf{Training Objective.}
The MetaAE is trained with two objectives. The latent reconstruction loss is defined as
\begin{equation}\label{eq:loss_z}
\mathcal{L}_{\mathrm{rec}} = \|\hat{\mathbf{z}}(\ell) - \mathbf{z}\|_2^2,
\end{equation}
and the ground-truth error at level $\ell$ is defined as
\begin{equation}
e(\ell) = \left\|\mathrm{sg}\!\left(\hat{\mathbf{z}}(\ell)\right) - \mathbf{z}\right\|_2^2,\end{equation}
where $\mathrm{sg}(\cdot)$ denotes the stop-gradient operator, and the error prediction loss is defined as
\begin{equation}
\mathcal{L}_{\mathrm{err}} = \|\hat{e}(\ell) - e(\ell)\|_2^2.
\end{equation}
For notational simplicity, $\hat{\mathbf{z}}(\ell)$ denotes the decoder output tokens excluding the error token.
The overall training objective is given by
\begin{equation}
\mathcal{L}_{\mathrm{total}} = \mathcal{L}_{\mathrm{rec}} + \lambda_{\mathrm{err}} \cdot \mathcal{L}_{\mathrm{err}},
\end{equation}
where $\lambda_{\mathrm{err}}$ balances reconstruction accuracy and quality estimation. In this work, $\lambda_{\mathrm{err}} = 1.0$ is chosen empirically.

\subsection{Semantic Quality-Aware Transmission Control}
Based on the Meta AutoEncoder, we define a semantic quality metric to guide adaptive transmission. In Eq.~\eqref{eq:err_pred}, the decoder predicts a reconstruction error $\hat{e}(\ell)$ for a given transmission level $\ell$. We map this value to a bounded quality score as
\begin{equation}\label{eq:quality_def}
q(\ell) = \exp\big(-c\cdot\hat{e}(\ell)\big), \qquad q(\ell) \in [0,1],
\end{equation}
where $c>0$ is a scaling hyperparameter. A higher $q(\ell)$ corresponds to a lower predicted reconstruction error and thus indicates a higher estimated reconstruction quality. The exponential mapping provides a monotonic transformation of the predicted reconstruction error into a normalized score that can be directly used for transmission control.

\vspace{1mm}
\noindent \textbf{Semantic-Aware Quality Metric.}
The proposed quality metric is inherently \emph{semantic-aware}. Specifically, the reconstruction target $\mathbf{z}$ is the latent representation extracted from a pretrained vision encoder with frozen weights. As such, $\mathbf{z}$ encodes high-level semantic information rather than low-level pixel details. Minimizing the reconstruction error in this latent space is therefore analogous to perceptual loss~\cite{johnson2016perceptual}, capturing semantic fidelity.
We further validate this design in Sec.~\ref{sec:results}, where we show that the proposed semantic quality $q(\ell)$ is highly correlated with downstream task performance (e.g., caption similarity), demonstrating that it serves as an effective surrogate for task-level performance in transmission control.

\vspace{1mm}
\noindent \textbf{Adaptive Transmission Control.}
The semantic quality score is used to drive a closed-loop transmission process between the edge and the cloud. Given a predefined quality threshold $\epsilon \in [0,1]$, the cloud evaluates $q(\ell)$ after each transmission step and determines whether additional information is required.

Specifically, starting from an initial transmission level $\ell_0$, the system iteratively performs:
\begin{enumerate}
    \item [1.]\textbf{Edge $\rightarrow$ Cloud:} transmit the next-level token;
    \item [2.]\textbf{Cloud:} reconstruct $\hat{\mathbf{z}}(\ell)$ and compute $q(\ell)$;
    \item [3.]\textbf{Decision:} if $q(\ell) \ge \epsilon$ \textbf{or} $\ell = K$, terminate; otherwise, request further transmission and return to Step~1.
\end{enumerate}

This process forms a feedback-driven control loop, enabling the system to adaptively balance communication cost and semantic fidelity. 
Fig.~\ref{fig:ctrl} illustrates the corresponding sequential diagram. 
In practice, the encoding process is performed only once at the initial iteration to generate the full ordered token sequence. The resulting tokens corresponding to all transmission levels are stored at the edge-side controller, and subsequent iterations simply transmit additional tokens from this cached representation without recomputing the encoder.
\begin{figure}[h]
    \centering
\includegraphics[width=0.55\columnwidth]{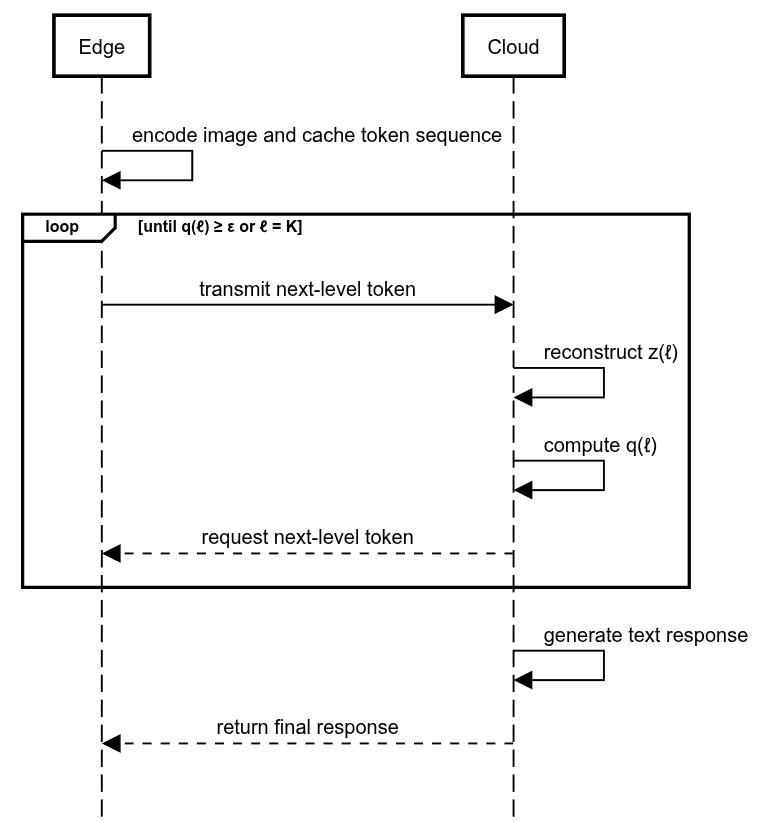}
    \caption{Progressive transmission with semantic quality-based control. The initial transmission of textual input is omitted in this diagram for simplicity.}
    \label{fig:ctrl}
\end{figure}

\section{Experimental Setup}\label{sec:exp}

\subsection{End-to-End Edge–Cloud Testbed}
\noindent \textbf{Edge.}
The edge device is built on the NXP i.MX95 application processor, which integrates a six-core ARM Cortex-A55 64-bit CPU and the Neutron NPU\footnote{https://www.nxp.com/products/i.MX95}, delivering up to 2 TOPS of inference throughput. The NPU supports hardware-accelerated execution of both CNN and transformer-based architectures.
The platform runs a Yocto Project-based Linux distribution with kernel version Linux kernel 6.12. Although commercially available i.MX95-based boards such as  IMX95-EVK was evaluated during the prototyping phase, we ultimately designed a custom SMARC\footnote{https://sget.org/standards/smarc/} system-on-module based on the i.MX95 platform. This allows us to fully leverage the favorable NPU power envelope and to incorporate additional interfaces such as CAN bus, RS-485 and Micron LPDDR5 memory, specifically targeting robotic platform development, where sustained low-power edge inference is a primary design constraint.

\vspace{1mm}
\noindent \textbf{Cloud.}
The cloud side is implemented on a server equipped with an Intel Core i7-10700K CPU, 32\,GB of RAM, and an NVIDIA RTX 2080 Super GPU. The system is developed using Python 3.11.9 and PyTorch 2.5.1. The RESTful API is built using Flask 3.0.3, which handles bidirectional communication between the edge and the cloud.
To emulate bandwidth-constrained network conditions, the cloud service is deployed within a Docker container, and traffic control is applied using the Linux \texttt{tc} utility with a token bucket filter (TBF). The configuration parameters are summarized in Table~\ref{tab:tc_config}.
\begin{table}[h]
\centering
\caption{Traffic control (tc) configuration for cloud side}
\label{tab:tc_config}
\begin{tabular}{|l|l|}
\hline
\textbf{Parameter} & \textbf{Value} \\
\hline
Interface & \texttt{<container ID>} \\
Queue discipline & \texttt{tbf} \\
Rate & \texttt{1 mbit} \\
Burst & \texttt{32 kbit} \\
Latency & \texttt{200 ms} \\
\hline
\end{tabular}
\end{table}

\vspace{1mm}
\noindent \textbf{VLM.}
To evaluate the proposed framework under resource-constrained settings, we adopt \emph{SmolVLM-256M}~\cite{marafioti2025smolvlm} with $4$-bit quantization\footnote{https://huggingface.co/HuggingFaceTB/SmolVLM-256M-Instruct} as the backbone VLM. This model is designed for lightweight deployment, with a total parameter count of $\mathord{\sim}256$ million.
The SmolVLM architecture follows a standard vision--language paradigm, consisting of a vision encoder that extracts visual features, and a language model that performs reasoning and generation. 
In our setup, the vision encoder is deployed on the edge device, while the language model is hosted on the cloud. 
During training, the VLM parameters are kept frozen, and only the MetaAE is optimized.

\vspace{1mm}
\noindent \textbf{MetaAE.}
The encoder consists of $2$ Transformer encoder layers with $8$ attention heads, hidden dimension of $576$, and an MLP dimension of $4\times576$. The decoder follows a similar design with $2$ Transformer encoder layers.
During training, progressive transmission is simulated by random prefix masking with four candidate transmission levels, corresponding to LTL ratios of $25$\%, $50$\%, $75$\%, and $100$\% of the full latent sequence. Tokens beyond the sampled prefix are zero-masked. 
We use a batch size of $128$ and train for $50$ epochs with the Adam optimizer and a learning rate of $1\times10^{-4}$. The model is trained on $50,000$ images from the COCO $2017$ training set~\cite{lin2014microsoft} without using annotations.
\subsection{Evaluation Metrics}
Due to the limited capacity of the SmolVLM-256M model, absolute task accuracy is not the main focus of our evaluation. Instead, we assess how well the proposed method preserves the \textit{Semantic Consistency} (SC) of the original model outputs under progressive transmission. In other words, the original model output is treated as the reference, and we measure the deviation introduced by the proposed framework.
We evaluate the method on two well-established benchmarks for VLMs:
\begin{itemize}
    \item COCO Captioning (MS COCO 2017)\cite{lin2014microsoft}
    \item POPE (Polling-based Object Probing Evaluation)~\cite{li2023evaluating}
\end{itemize}
For each dataset, we randomly sample $1,000$ instances from the test split, and conduct $10$ independent runs to account for stochasticity in generation. All reported results correspond to the average over these runs.
Moreover, SC is quantified using a sentence embedding tool from fastembed\footnote{https://github.com/qdrant/fastembed}. Specifically, we compute the \textit{cosine similarity} between the embedding of answers generated by the proposed method and by the original model. This metric captures the degree to which the semantic meaning is preserved, even if the exact wording differs.

In addition, we compare the following $3$ deployment configurations: ($i$) \textit{Full Edge}, where the entire VLM is executed on the edge; ($ii$) \textit{Full Cloud}, where all computation is done on a cloud server; and ($iii$) \textit{Edge-Cloud}, corresponding to the proposed progressive transmission framework that partitions computation between edge and cloud.
For a fair comparison, we do not apply any entropy-based compression to the transmitted features (or pixels) in all above configurations. More comprehensive evaluation details are provided in Appendix~\ref{app:extended_benchmarks}.

\section{Results and Discussion}\label{sec:results}
Fig.~\ref{fig:acc_delay} illustrates the trade-off between Semantic Consistency (SC) and transmission delay over different LTL ratios, with SC and delay plotted on the left and right y-axes, respectively. Starting from $100$\% LTL ratio, both benchmarks achieve their highest SC. As the LTL ratio gradually reduces, SC degrades gracefully, remaining above $80$\% even at aggressive compression levels (e.g., $25$\%), while the transmission delay decreases almost linearly.
Notably, at $25$\% LTL ratio, the proposed method achieves a substantial reduction in delay ($75$\%), while incurring less than $5$\% degradation in SC. 
Additionally, Fig.~\ref{fig:corr} shows a near-perfect monotonic relationship between the proposed quality metric defined in Eq.~\eqref{eq:quality_def} and COCO task SC, with a Spearman's Rank Correlation Coefficient (SRCC) of $1.0$. This indicates that the proposed quality metric serves as a reliable proxy for semantic accuracy, making it well-suited for guiding adaptive transmission decisions.
\begin{figure}[h]
    \centering
    \begin{subfigure}{0.32\textwidth}
        \centering
        \includegraphics[width=\textwidth, trim=0.6cm 0cm 0.6cm 0.4cm, clip]{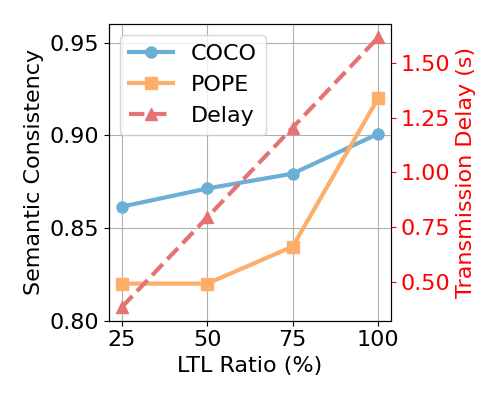}
        \caption{SC and delay curves.}
        \label{fig:acc_delay}
    \end{subfigure}
    \hspace*{2cm}
    \begin{subfigure}{0.25\textwidth}
        \centering
        \includegraphics[width=\textwidth, trim=0.6cm 0cm 0.6cm 0.4cm, clip]{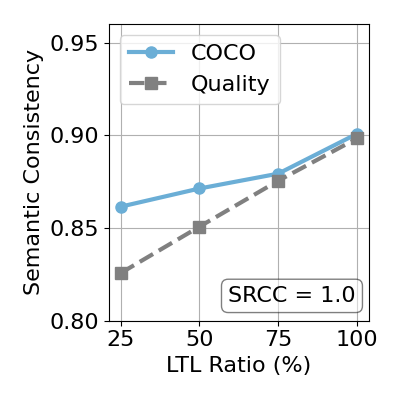}
        \caption{SC vs Quality.}
        \label{fig:corr}
    \end{subfigure}
    \caption{SC-delay trade-off and SC-quality correlation.}
\end{figure}
Table~\ref{tab:qualitative} shows qualitative results of image descriptions under different LTL ratios. For this representative real-world example, no noticeable semantic deviations are observed across levels: the main scene content is described correctly in all cases, while the wording varies with the LTL ratio. The image was taken by the authors and is therefore not included in the training/validation set, which further suggests robust generalization to Out-Of-Distribution (OOD), real-world data.
\begin{table}[ht]
\vspace*{0.04in}
\caption{Answers to the question "Describe the image in detail." under different LTL ratios.}\label{tab:qualitative}
\centering
\begin{tabular}{|>{\centering\arraybackslash}m{1.6cm}|
                >{\centering\arraybackslash}m{2.5cm}|
                >{\centering\arraybackslash}m{2.5cm}|
                >{\centering\arraybackslash}m{2.5cm}|
                >{\centering\arraybackslash}m{2.5cm}|
                >{\centering\arraybackslash}m{2.5cm}|}
\hline
Image & 25\% & 50\% & 75\% & 100\% & ground truth\\
\hline
\raisebox{-0.05\height}{\includegraphics[width=1.2cm]{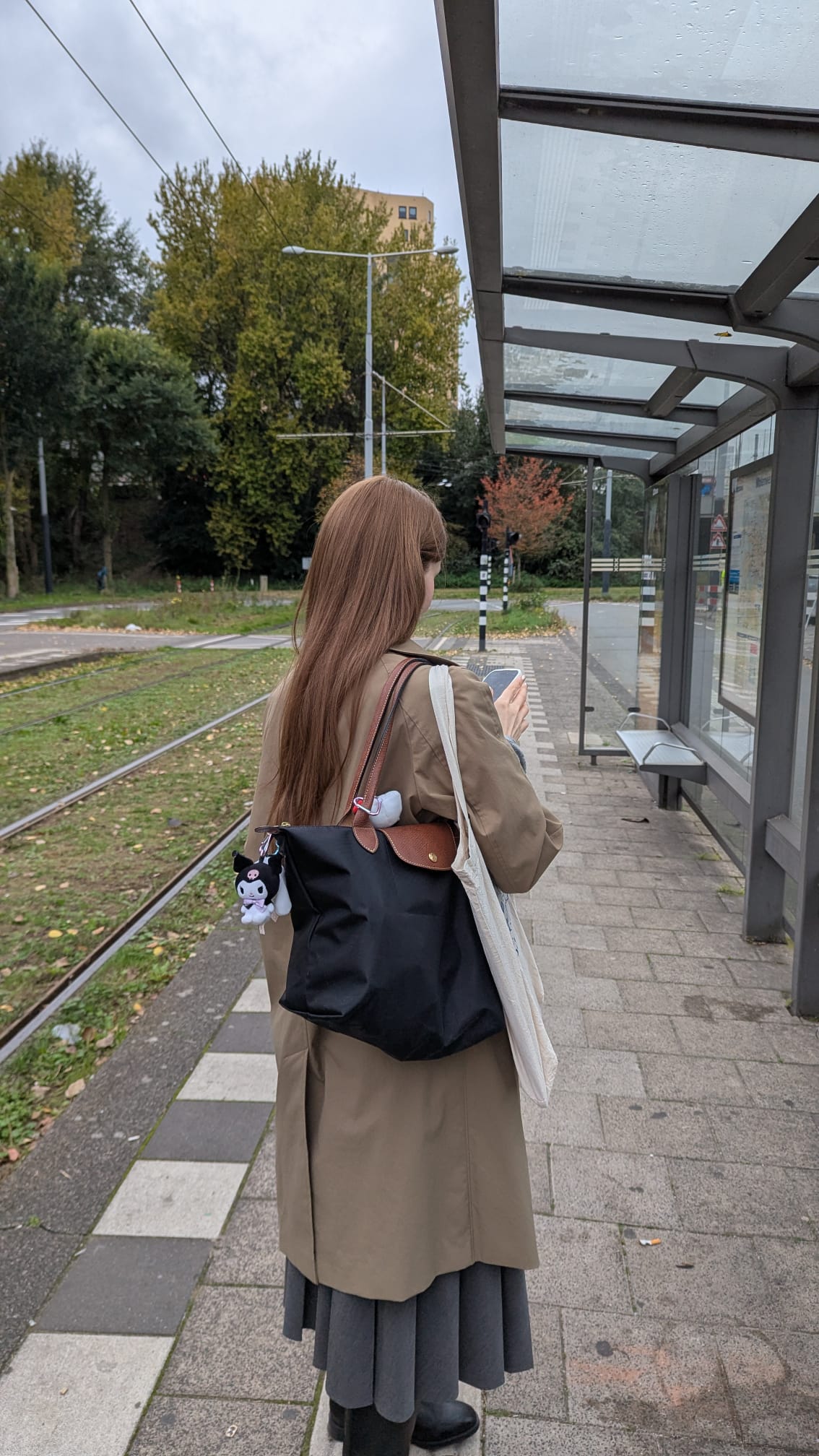}} & A person standing in front of a bus stop with a brown coat on the sidewalk in front of the bus stop... & A woman standing in front of a bus stop with a brown jacket on and a black bag with a leather strap... & A woman standing in front of a bus stop with a brown jacket on and a black bag with some red... & A woman standing in a bus stop, wearing a tan trench coat, brown skirt, and carrying a black handbag... & A woman standing in front of a bus stop, wearing a tan trench coat, brown skirt...\\
\hline
\raisebox{-0.05\height}{\includegraphics[width=1.2cm]{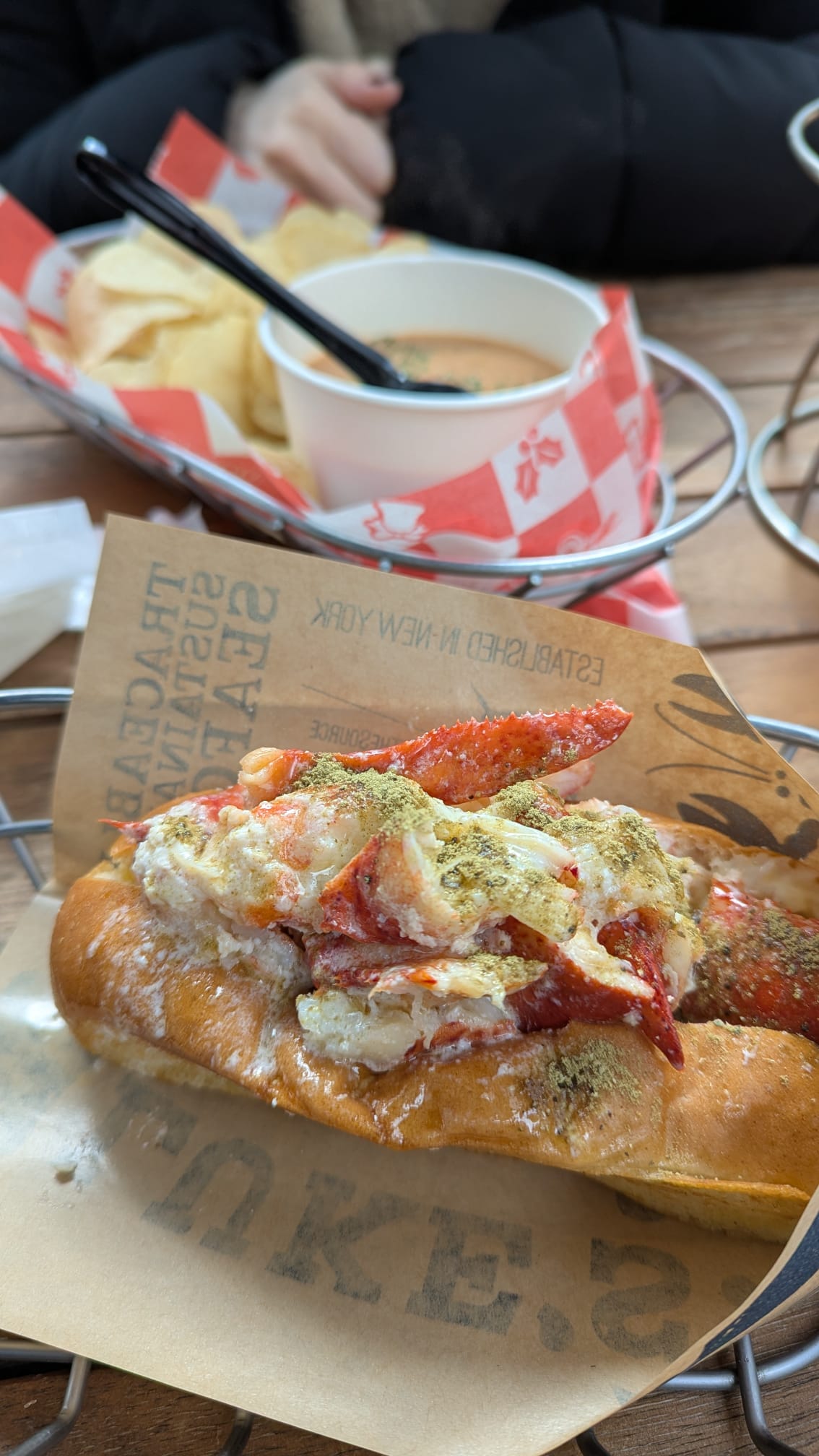}} & A sandwich sits on a table, in a basket of what looks like some kind of bread or pastry... & A sandwich sits on a table, in a basket of what looks like some kind of soup... & A close-up shot of a white and red mini-sauce sandwich, possibly a pizza or a meatball... & A small sandwich on a table, in the background there is a bowl of what looks like soup... & A sandwich sits on a table, in a basket of what looks like a soup...\\
\hline
\raisebox{-0.05\height}{\includegraphics[width=1.8cm]{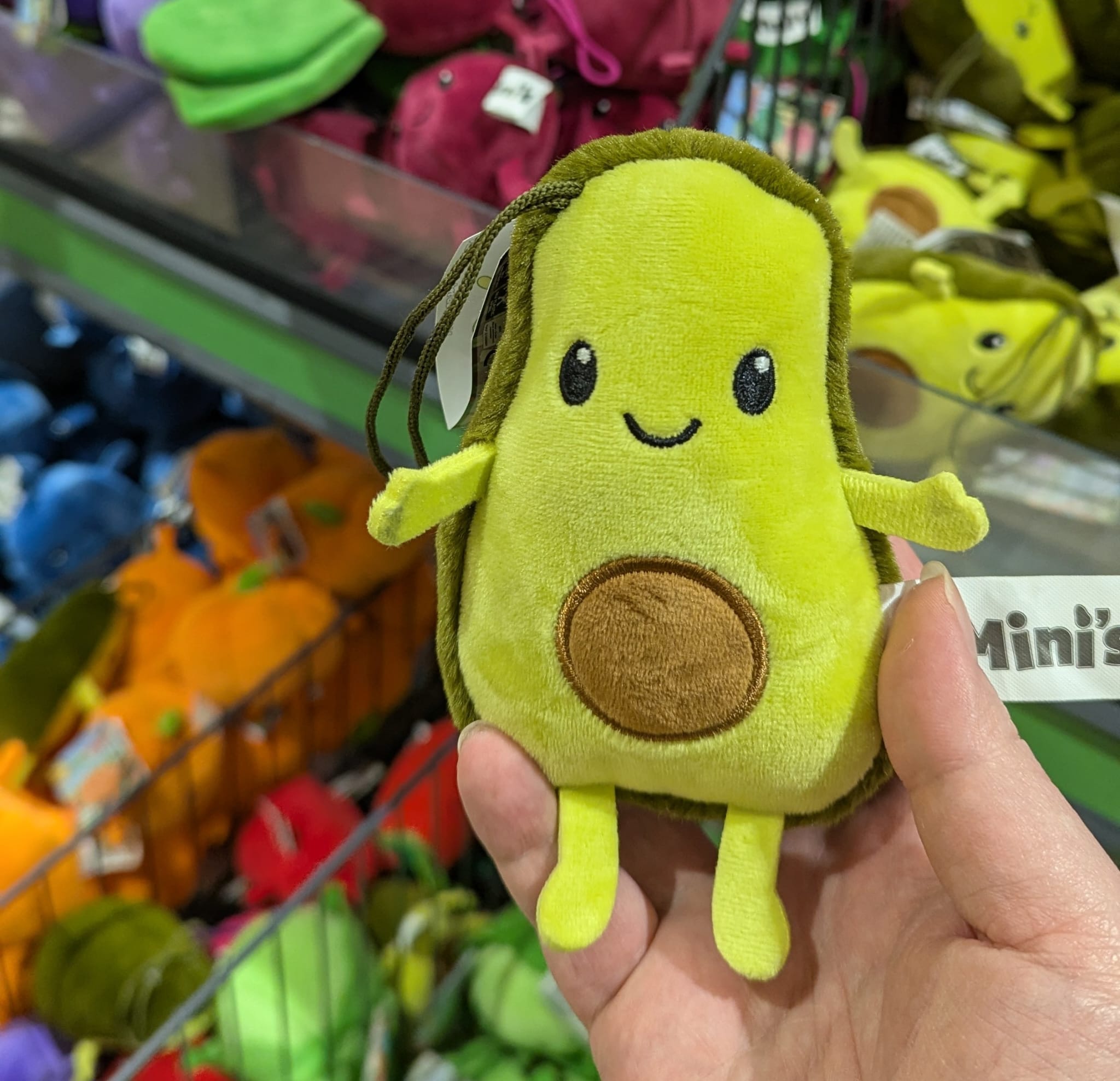}} & In the foreground, a person's hand is holding a plush toy with a small face and a small body... & In the foreground, a person's hand is holding a plush avocado-shaped stuffed animal... & In the foreground, a person's hand is holding a small yellow and brown stuffed animal... & In the foreground, a person is holding a small yellow and brown plush toy... & In the foreground, a person is holding a small yellow avocado with a smiling face...\\
\hline
\end{tabular}
\end{table}
As shown in Table~\ref{tab:compare}, the Full Edge configuration is computation-bound, while the Full Cloud setting is dominated by network delay. In contrast, the proposed Edge-Cloud approach achieves the lowest end-to-end delay by balancing computation and communication. Notably, the overhead of MetaAE is marginal in both delay and energy cost ($0.35$\%), while contributing to a sharp reduction in transmission delay.
\begin{table}[h]
\centering
\caption{System-Level Performance at $1$Mbps uplink. The Edge-Cloud setting is with $100$\% LTL ratio, indicating the worst case scenario for the proposed progressive scheme.}\label{tab:compare}
\label{tab:system_performance}
\begin{tabular}{|l|c|c|c|}
\hline
\textbf{Metric} & \textbf{Full Edge} & \textbf{Full Cloud} & \textbf{Edge--Cloud} \\
\hline
E2E Delay (s) & 7.98 & 9.32 & 6.94 \\
\hline
Vision Encode (s) & 3.20 (edge) & 0.05 (cloud) & 3.21 (edge) \\
Network Delay (s) & - & 8.75 & 1.62 \\
LLM Decode (s) & 3.46 (edge) & 0.49 (cloud) & 0.49 (cloud) \\
MetaAE Encode (s) & - & - & 0.01 (edge) \\
\hline
Peak Memory (MB) & 625 (edge) & 92 (edge) & 344 (edge) \\
NPU Energy Cost (J) & 23.31 & - & 11.27 \\
\hline
\end{tabular}
\end{table}

\section{Conclusion}\label{sec:conclude}
In this paper, we proposed a progressive semantic communication framework for efficient edge–cloud deployment of VLMs. By introducing a plug-and-play MetaAE with progressively refinable latent representations, the system enables partial information reuse and significantly improving communication efficiency. The integration of a semantic quality-aware transmission control mechanism further provides adaptive and fine-grained control over the trade-off between communication cost and semantic fidelity.
Through a real-world end-to-end testbed on an embedded edge platform and a cloud server, we demonstrated that the proposed approach achieves substantial reductions in transmission latency ($\mathord{\sim}75$\%) while maintaining high semantic consistency ($\mathord{>}80$\%), under constrained bandwidth conditions. These results highlight the effectiveness of progressive semantic communication as a key enabler for deploying VLMs in resource-limited edge–cloud scenarios.
Future work will evaluate broader VLMs under dynamic network conditions, while complementing semantic consistency with direct task-level accuracy metrics and comparisons against alternative compression and transmission schemes.

\section*{Acknowledgment}
This research was partially funded by the Dutch 6G flagship project Future Network Services (FNS).

\newpage
\bibliographystyle{IEEEtran}
\bibliography{biblio}

\newpage
\appendices
\input{appendix_}
\newpage
\input{appendix}

\end{document}

%% file: appendix_.tex
\section{Analysis of Progressive Refinability}
\label{app:progressive-proof}
We now show why the proposed prefix-masked
training strategy encourages progressively refinable latent representations.
Although the coarse-to-fine structure is not imposed by explicit supervision,
prefix masking introduces a structured optimization bias over the ordered token sequence. Specifically, earlier tokens are visible under more transmission
levels and therefore participate in more reconstruction objectives than later
tokens. Proposition~\ref{prop:asymmetric-pressure} formalizes this asymmetric optimization pressure and shows
that prefix masking naturally prioritizes earlier tokens during training.
Beyond this optimization bias, progressive refinement is also supported by the
nested structure of prefix transmission. A higher transmission level contains
all tokens available at lower levels, together with additional refinement tokens.
Therefore, under sufficient decoder capacity, the best achievable reconstruction
error cannot increase as the prefix length grows. Proposition~\ref{prop:nested} formalizes this
nested-prefix argument and establishes that the optimal reconstruction risk is
non-increasing with the transmission level. Together, these results explain why
prefix-masked training encourages early tokens to encode broadly useful coarse
information, while later tokens provide incremental refinements.

\begin{proposition}[Prefix masking induces asymmetric optimization pressure]
\label{prop:asymmetric-pressure}
Let $z \in \mathbb{R}^{d}$ denote the target latent representation and let
$r=f_{\mathrm{enc}}(z)=(r_1,\ldots,r_K)$ be the ordered token sequence produced
by the meta encoder. Let $P_\ell(r)$ denote the prefix-masking operator that
keeps the first $\ell$ tokens and masks the remaining tokens:
\[
P_\ell(r)=(r_1,\ldots,r_\ell,0,\ldots,0).
\]
Consider the prefix-masked reconstruction objective
\[
\mathcal{L}_{\mathrm{pm}}
=
\frac{1}{K}
\sum_{\ell=1}^{K}
\mathcal{L}_{\ell},
\qquad
\mathcal{L}_{\ell}
=
\left\|
f_{\mathrm{dec}}(P_\ell(r))-z
\right\|_2^2 .
\]
Then token $r_i$ contributes only to the reconstruction losses whose
transmission levels satisfy $\ell \ge i$. Consequently, token $r_i$ receives
optimization signals from exactly $K-i+1$ prefix levels. Hence, earlier tokens
receive reconstruction gradients from more prefix levels than later tokens,
which induces an asymmetric optimization pressure over token positions.
\end{proposition}

\begin{proof}
For a fixed transmission level $\ell$, the prefix-masking operator $P_\ell$
keeps token $r_i$ if and only if $i \le \ell$. If $i>\ell$, then $r_i$ is
replaced by a zero padding token and does not affect the decoder input
$P_\ell(r)$. Therefore, the level-$\ell$ reconstruction loss
\[
\mathcal{L}_{\ell}
=
\left\|
f_{\mathrm{dec}}(P_\ell(r))-z
\right\|_2^2
\]
does not depend on $r_i$ whenever $\ell<i$. Hence,
\[
\nabla_{r_i}\mathcal{L}_{\ell}=0,
\qquad
\text{for all } \ell<i .
\]

Taking the gradient of the total prefix-masked objective with respect to
$r_i$ gives
\[
\nabla_{r_i}\mathcal{L}_{\mathrm{pm}}
=
\nabla_{r_i}
\left(
\frac{1}{K}
\sum_{\ell=1}^{K}
\mathcal{L}_{\ell}
\right)
=
\frac{1}{K}
\sum_{\ell=1}^{K}
\nabla_{r_i}\mathcal{L}_{\ell}.
\]
Using $\nabla_{r_i}\mathcal{L}_{\ell}=0$ for all $\ell<i$, we obtain
\[
\nabla_{r_i}\mathcal{L}_{\mathrm{pm}}
=
\frac{1}{K}
\sum_{\ell=i}^{K}
\nabla_{r_i}\mathcal{L}_{\ell}.
\]
Thus token $r_i$ receives reconstruction gradients only from the losses
corresponding to levels $\ell=i,i+1,\ldots,K$, i.e., from exactly $K-i+1$
prefix levels.

For two token indices $i<j$, we have
\[
K-i+1 > K-j+1.
\]
Therefore, earlier tokens are involved in more reconstruction objectives than
later tokens. This proves that prefix masking induces an asymmetric optimization
pressure over the ordered token sequence.
\end{proof}

\begin{proposition}[Nested prefixes encourage progressive refinement]
\label{prop:nested}
Let the level-$\ell$ reconstruction risk be defined as
\[
\mathcal{R}_{\ell}
=
\mathbb{E}_{z}
\left[
\left\|
f_{\mathrm{dec}}(P_\ell(f_{\mathrm{enc}}(z)))-z
\right\|_2^2
\right],
\]
where $P_\ell$ keeps the first $\ell$ tokens and masks the remaining tokens.
Let $\mathcal{R}_{\ell}^{\star}$ denote the minimum achievable reconstruction
risk at transmission level $\ell$ over the admissible encoder and decoder
function classes. If the decoder at level $\ell+1$ is sufficiently expressive
to ignore the additional token $r_{\ell+1}$, then the optimal achievable
reconstruction risk is non-increasing with the transmission level:
\[
\mathcal{R}_{\ell+1}^{\star}
\le
\mathcal{R}_{\ell}^{\star},
\qquad
\ell=1,\ldots,K-1.
\]
Therefore, increasing the prefix length can only improve, or at least preserve,
the best achievable reconstruction fidelity.
\end{proposition}

\begin{proof}
At transmission level $\ell$, the decoder observes the prefix-masked
representation
\[
P_\ell(r)
=
(r_1,\ldots,r_\ell,0,\ldots,0).
\]
At transmission level $\ell+1$, the decoder observes
\[
P_{\ell+1}(r)
=
(r_1,\ldots,r_\ell,r_{\ell+1},0,\ldots,0).
\]
Thus, the information available at level $\ell$ is contained in the information
available at level $\ell+1$. The level-$(\ell+1)$ representation contains all
tokens used at level $\ell$, together with one additional token $r_{\ell+1}$.

Let $h_\ell$ denote an arbitrary reconstruction mapping that uses the level-$\ell$
prefix:
\[
\hat z^{(\ell)} = h_\ell(P_\ell(r)).
\]
Because $P_{\ell+1}(r)$ contains $P_\ell(r)$ as a subset of its unmasked tokens,
a decoder operating at level $\ell+1$ can reproduce the same reconstruction
mapping by ignoring the additional token $r_{\ell+1}$. That is, there exists a
mapping $h_{\ell+1}$ such that
\[
h_{\ell+1}(P_{\ell+1}(r))
=
h_\ell(P_\ell(r)).
\]
Therefore, any reconstruction function achievable at level $\ell$ is also
achievable at level $\ell+1$. Hence, the admissible reconstruction function
class at level $\ell$ is contained in the admissible reconstruction function
class at level $\ell+1$.

Since minimizing over a larger function class cannot yield a higher minimum
risk, we have
\[
\mathcal{R}_{\ell+1}^{\star}
=
\inf_{h_{\ell+1}}
\mathbb{E}_{z}
\left[
\left\|
h_{\ell+1}(P_{\ell+1}(r))-z
\right\|_2^2
\right]
\le
\inf_{h_{\ell}}
\mathbb{E}_{z}
\left[
\left\|
h_{\ell}(P_{\ell}(r))-z
\right\|_2^2
\right]
=
\mathcal{R}_{\ell}^{\star}.
\]
Thus,
\[
\mathcal{R}_{\ell+1}^{\star}
\le
\mathcal{R}_{\ell}^{\star},
\qquad
\ell=1,\ldots,K-1.
\]
This proves that the optimal achievable reconstruction error is non-increasing
as more prefix tokens are made available. Consequently, nested prefix
transmission encourages progressive refinement of the reconstruction.
\end{proof}

%% file: appendix.tex
%
%




\section{Extended Measurements and Per-Run Benchmarks}
\label{app:extended_benchmarks}
This appendix reports the complete measurement record underlying
Table~III of the main text.
Ten independent runs were performed for each of three deployment
configurations---Full Edge, Full Cloud, and the proposed
Edge--Cloud scheme---on the end-to-end testbed described in
Section~IV-A of the main text.
All cloud-side experiments are carried out under a 1-Mbps throttled
uplink emulated with the Linux \texttt{tc} token-bucket filter whose
parameters are listed in Table~I of the main text.
On-board power is measured with a GW INSTEK SPS-2415 programmable
DC supply at 11.7\,V, sampled frame-by-frame from a video recording
of the front panel during a full inference cycle.
Wall-clock latencies are recorded at the client and the server using
synchronized monotonic clocks;
\emph{server\_recv\_ms} denotes the time from the start of the
HTTP POST at the edge to the moment the server finishes reading
the request body.
Peak resident set size (RSS) is obtained via
\texttt{/proc/[pid]/status} sampled at 10\,Hz.
The implementation code will be released upon publication at \url{https://github.com/open-ep/ProSemComVLM}.

\subsection{Full-Edge Configuration}

In the Full-Edge configuration, the entire SmolVLM-256M pipeline
(vision encoder, LLM prefill, LLM decode) runs locally on the
i.MX95 SMARC module, with no network dependency.
The Neutron NPU executes the quantized vision encoder
(\texttt{vision\_encoder\_q4.onnx}), and the ONNX-exported LLM
decoder runs on the board through ONNX Runtime.
Each run generates 120 tokens for the input image of size $512\times512\times3$.

\begin{table}[h]
\centering
\caption{Per-Run Latency Breakdown for the Full-Edge Configuration}
\label{tab:full_edge}
\begin{tabular}{ccccccc}
\toprule
Run &
\begin{tabular}{c}Vision\\Encode (ms)\end{tabular} &
\begin{tabular}{c}LLM\\Prefill (ms)\end{tabular} &
TTFT (ms) &
\begin{tabular}{c}LLM\\Decode (ms)\end{tabular} &
\begin{tabular}{c}Total\\Generate (ms)\end{tabular} &
\begin{tabular}{c}E2E\\tok/s\end{tabular} \\
\midrule
1  & 3188.52 & 1322.39 & 4510.91 & 3416.58 & 7927.49 & 15.14 \\
2  & 3196.16 & 1333.97 & 4530.13 & 3412.34 & 7942.47 & 15.11 \\
3  & 3196.69 & 1313.17 & 4509.86 & 3495.68 & 8005.54 & 14.99 \\
4  & 3209.57 & 1310.73 & 4520.30 & 3463.95 & 7984.25 & 15.03 \\
5  & 3205.05 & 1317.80 & 4522.85 & 3403.12 & 7925.97 & 15.14 \\
6  & 3202.93 & 1326.81 & 4529.74 & 3633.72 & 8163.46 & 14.70 \\
7  & 3208.64 & 1318.79 & 4527.43 & 3540.80 & 8068.23 & 14.87 \\
8  & 3207.00 & 1314.58 & 4521.58 & 3351.90 & 7873.48 & 15.24 \\
9  & 3213.13 & 1316.15 & 4529.28 & 3496.96 & 8026.24 & 14.95 \\
10 & 3199.68 & 1313.74 & 4513.42 & 3389.38 & 7902.80 & 15.18 \\
\midrule
\textbf{Mean} &
  \textbf{3202.7} & \textbf{1318.8} & \textbf{4521.6} &
  \textbf{3460.4} & \textbf{7982.0} & \textbf{15.04} \\
\textbf{Std.} &
  $\pm$7.1 & $\pm$7.4 & $\pm$7.4 &
  $\pm$83.8 & $\pm$82.8 & $\pm$0.16 \\
\bottomrule
\end{tabular}
\end{table}

Vision encoding is remarkably deterministic on the Neutron NPU
(coefficient of variation, CV\,=\,0.22\%);
LLM prefill, which operates on a fixed-size visual token tensor,
is similarly stable (CV\,=\,0.56\%).
The dominant source of run-to-run variability is the auto-regressive
decode phase ($\pm$83.8\,ms), which is sensitive to on-board memory
pressure and scheduler state.
Time-to-first-token (TTFT), the user-perceived latency before any
output is produced, averages 4521.6\,ms; full 120-token delivery
takes 7982.0\,ms.
Peak RSS reaches 624.6\,MB, of which approximately 388\,MB is
attributable to the persistently loaded vision encoder and LLM
weights, and 156\,MB to the KV cache and intermediate activations
generated during a 120-token decode.
Fig.~\ref{fig:power_full_edge} shows the measured on-board power
profile during a Full-Edge cycle.
\begin{figure}[!h]
    \centering
\includegraphics[width=1.0\columnwidth]{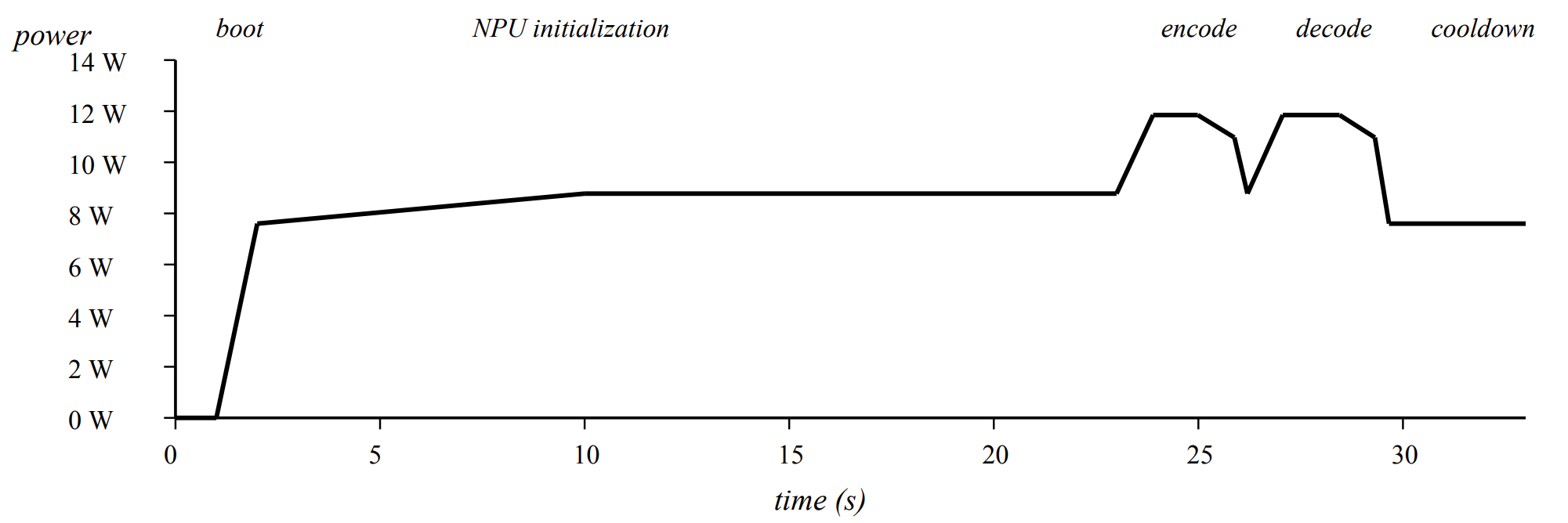}
    \caption{Measured on-board power during a single Full-Edge
    inference cycle, using the same supply, sampling, and power
    levels as Fig.~\ref{fig:power_edge_cloud}.
    Two consecutive compute bursts are visible, each sharing the
    same envelope shape (rise--hold--decay--drop) as the single burst
    in Fig.~\ref{fig:power_edge_cloud}: a first burst peaking at
    11.85\,W for the NPU vision encode (envelope width 3.20\,s,
    matching Table~\ref{tab:full_edge}) and a second slightly longer
    burst at the same peak for the auto-regressive LLM decode
    (envelope width 3.46\,s, matching Table~\ref{tab:full_edge}).
    The transition between the two bursts is a momentary drop to the
    8.775-W initialization baseline rather than a sustained plateau.
    Subtracting the 7.6-W idle baseline, the total NPU-attributable
    energy obtained by integrating $(P(t)-P_\mathrm{idle})$ over the
    two bursts is 23.31\,J.}
    \label{fig:power_full_edge}
\end{figure}
\subsection{Full-Cloud Configuration at 1\,Mbps}

In the Full-Cloud configuration, the edge performs no inference
and forwards the raw $512\times512\times3$ \texttt{uint8} image
to the server as a base64-encoded JSON payload of 1024.3\,KiB.
Both the vision encoder (\emph{Stage~1}) and the LLM decoder
(\emph{Stage~2}) then execute on the GPU server.
\begin{table}[h]
\centering
\caption{Per-Run Latency Breakdown at 1\,Mbps for the Full-Cloud Configuration (Pixels)}
\label{tab:full_cloud}
\begin{tabular}{cccccc}
\toprule
Run &
\begin{tabular}{c}server\_recv\\(ms)\end{tabular} &
\begin{tabular}{c}Stage 1\\(ms)\end{tabular} &
\begin{tabular}{c}Stage 2\\(ms)\end{tabular} &
\begin{tabular}{c}E2E\\(ms)\end{tabular} \\
\midrule
1  & 8748.0 &  61.8 & 511.6 & 9345.4 \\
2  & 8746.8 &  47.3 & 503.4 & 9335.4 \\
3  & 8748.4 &  49.8 & 505.5 & 9338.8 \\
4  & 8745.5 &  45.6 & 533.0 & 9354.2 \\
5  & 8766.5 &  46.1 & 462.0 & 9293.2 \\
6  & 8746.7 &  45.6 & 461.3 & 9265.9 \\
7  & 8787.0 &  45.6 & 464.4 & 9314.8 \\
8  & 8746.7 &  45.6 & 462.1 & 9273.8 \\
9  & 8751.2 &  45.7 & 505.3 & 9331.2 \\
10 & 8754.1 &  46.8 & 525.4 & 9361.3 \\
\midrule
\textbf{Mean} & \textbf{8754.1} & \textbf{48.0} &
  \textbf{493.4} & \textbf{9321.4} \\
\textbf{Std.} & $\pm$11.4 & $\pm$5.0 &
  $\pm$26.7 & $\pm$31.7 \\
\bottomrule
\end{tabular}
\end{table}
At 1\,Mbps, the theoretical transmission time of a 1024.3-KiB
payload is approximately 8395\,ms.
The measured \emph{server\_recv\_ms} mean of 8754\,ms
(CV\,=\,0.13\%) is within 4.3\% of this bound once HTTP framing,
TCP slow start, and the TBF burst/latency parameters are
accounted for.
Network transfer therefore accounts for 93.9\% of the
end-to-end latency; the combined server-side vision and language
computation (541\,ms) contributes less than 6\%.
Peak on-board RSS remains at 92.2\,MB since no model is loaded
locally; the 6-MB increment over the 86-MB baseline is the pixel
buffer (0.75\,MB), the base64-encoded payload ($\sim$1.3\,MB),
and the HTTP response buffer.

\subsection{Edge--Cloud Configuration at 1\,Mbps}

In the proposed Edge--Cloud configuration, the Neutron NPU
executes \texttt{vision\_encoder\_q4.onnx} on the board and
emits a latent tensor of shape $(64, 576)$ in \texttt{float32},
serialized as a 192.3-KiB JSON payload.
The server receives the latent, performs the LLM decode, and
returns the text response.
The full progressive transmission pipeline with the Meta
AutoEncoder sits between the vision encoder and the HTTP layer;
the measurements reported here correspond to transmission at
100\% latent transmission level (LTL ratio).
Fig.~\ref{fig:power_edge_cloud} shows the measured on-board
power profile during an Edge--Cloud cycle.
\begin{figure}[h]
    \centering
\includegraphics[width=1.0\columnwidth]{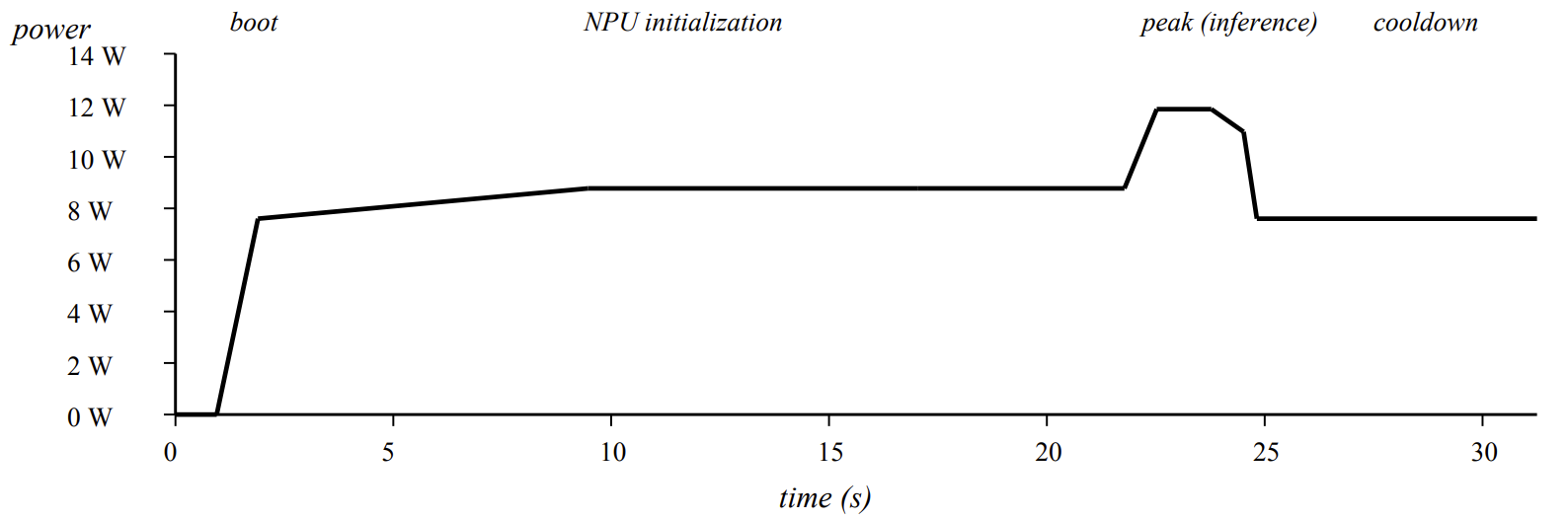}
    \caption{Measured on-board power during a single Edge--Cloud
    inference cycle (supply voltage 11.7\,V, sampled
    frame-by-frame from the DC power supply panel).
    Four operating regimes are visible: boot (rising from 0 to
    7.6\,W), NPU initialization (holding at 8.775\,W), a peak
    inference burst reaching 11.85\,W as the latent is transmitted
    to the cloud and decaying to 10.97\,W (envelope width 3.21\,s,
    matching the mean NPU encoding time of
    Table~\ref{tab:edge_cloud}), and a cooldown back to 7.6\,W.
    Subtracting the 7.6-W idle baseline, the NPU-attributable
    energy obtained by integrating $(P(t)-P_\mathrm{idle})$ over
    the burst is 11.27\,J.}
    \label{fig:power_edge_cloud}
\end{figure}
\begin{table}[h]
\centering
\caption{Per-Run Latency and Board Energy at 1\,Mbps for Edge--Cloud Configuration ($100$\% LTL ratio)}
\label{tab:edge_cloud}
\begin{tabular}{cccccc}
\toprule
Run &
\begin{tabular}{c}NPU\\Encode (ms)\end{tabular} &
\begin{tabular}{c}server\_recv\\(ms)\end{tabular} &
\begin{tabular}{c}LLM\\Decode (ms)\end{tabular} &
\begin{tabular}{c}E2E\\(ms)\end{tabular} &
\begin{tabular}{c}Enc. Energy\\(J)\end{tabular} \\
\midrule
1  & 3187.4 & 1618.4 & 529.1 & 6953.6 & 11.19 \\
2  & 3165.3 & 1617.7 & 487.1 & 6889.3 & 11.11 \\
3  & 3191.0 & 1617.7 & 480.1 & 6908.1 & 11.20 \\
4  & 3166.2 & 1617.7 & 485.4 & 6888.7 & 11.12 \\
5  & 3578.8 & 1616.8 & 480.4 & 7293.5 & 12.57 \\
6  & 3167.6 & 1617.3 & 476.4 & 6880.3 & 11.12 \\
7  & 3161.0 & 1617.3 & 484.3 & 6881.5 & 11.10 \\
8  & 3166.1 & 1617.4 & 483.7 & 6886.2 & 11.12 \\
9  & 3169.2 & 1617.7 & 484.0 & 6890.1 & 11.13 \\
10 & 3155.1 & 1617.0 & 484.4 & 6874.0 & 11.08 \\
\midrule
\textbf{Mean} &
  \textbf{3210.8} & \textbf{1617.7} & \textbf{487.5} &
  \textbf{6934.5} & \textbf{11.27} \\
\textbf{Std.} &
  $\pm$123.1 & $\pm$0.4 & $\pm$14.2 &
  $\pm$121.6 & $\pm$0.46 \\
\bottomrule
\end{tabular}
\end{table}

Excluding Run~5 (3578.8\,ms, attributable to transient OS
scheduling interference. See Table~\ref{tab:edge_cloud}), NPU encoding is tightly bounded:
the remaining nine runs have a mean of 3171\,ms and a standard
deviation of only $\pm$11\,ms (CV\,=\,0.3\%).
Network transfer is essentially deterministic
(CV\,=\,0.03\%), which is expected because the 192.3-KiB latent
payload plus TBF shaping yields a repeatable 1.62-s upload.
The end-to-end composition is therefore 46.3\% on-board NPU
encode, 23.3\% network, 7.0\% cloud LLM decode, with the
remainder attributable to protocol and scheduling overhead.
Peak RSS is 343.7\,MB, of which 254.6\,MB is the
vision-encoder ONNX session.

\subsection{Batch-Size Ablation of the Vision Encoder}

To isolate the effect of back-end choice on the vision encoder,
we sweep the batch size from 1 to 8 on both the CPU and the
Neutron NPU, with \texttt{max\_new\_tokens} fixed at 64.
The input image is replicated within each batch, so
that vision encoding is the only variable under test.
For each configuration, we report the per-image vision-encoder
time (Vis/img), the total generation time (Gen), the aggregate
throughput, and peak RSS.

\begin{table}[h]
\centering
\caption{Performance of CPU vs.\ Neutron NPU on Various Batch Sizes}
\label{tab:batch}
\setlength{\tabcolsep}{4pt}
\begin{tabular}{c|cccccc|cccccc}
\toprule
\multirow{2}{*}{BS} &
\multicolumn{5}{c}{CPU} & &
\multicolumn{5}{c}{NPU} \\
\cmidrule(lr){2-6} \cmidrule(lr){8-12}
 &
\begin{tabular}{c}Vis/img\\(ms)\end{tabular} &
\begin{tabular}{c}TTFT\\(ms)\end{tabular} &
\begin{tabular}{c}Gen\\(ms)\end{tabular} &
\begin{tabular}{c}Thrpt\\(tok/s)\end{tabular} &
\begin{tabular}{c}RSS\\(MB)\end{tabular} & &
\begin{tabular}{c}Vis/img\\(ms)\end{tabular} &
\begin{tabular}{c}TTFT\\(ms)\end{tabular} &
\begin{tabular}{c}Gen\\(ms)\end{tabular} &
\begin{tabular}{c}Thrpt\\(tok/s)\end{tabular} &
\begin{tabular}{c}RSS\\(MB)\end{tabular} \\
\midrule
1 & 9029.8 & 10554.6 & 5892.4 & 10.86 & 653.1 &
  & 3198.6 &  4469.2 & 4385.4 & 14.59 & 600.0 \\
2 & 9004.1 & 21028.9 & 8183.1 & 15.64 & 876.5 &
  & 2644.8 &  7669.8 & 6668.6 & 19.19 & 802.7 \\
3 & 9127.3 & 31933.5 & 9459.6 & 20.30 & 1085.4 &
  & 2482.0 & 10953.0 & 6754.1 & 28.43 & 1015.4 \\
4 & 9425.0 & 43871.5 & 10022.5 & 25.54 & 1195.7 &
  & 2368.3 & 14082.1 & 7420.5 & 34.50 & 1121.0 \\
5 & 9519.8 & 54841.1 & 11679.7 & 27.40 & 1297.0 &
  & 2311.6 & 17194.4 & 8218.2 & 38.94 & 1231.1 \\
6 & 9437.6 & 65211.1 & 12717.0 & 30.20 & 1463.4 &
  & 2327.7 & 20662.1 & 8980.8 & 42.76 & 1365.2 \\
7 & 9683.1 & 78767.9 & 15515.6 & 28.87 & 1582.5 &
  & 2251.4 & 23698.0 & 10236.0 & 43.77 & 1537.5 \\
8 & 9485.3 & 88417.5 & 16395.0 & 31.23 & 1820.6 &
  & 2234.1 & 26866.5 & 10688.1 & 47.90 & 1787.4 \\
\bottomrule
\end{tabular}
\end{table}

Three patterns are consistent across the sweep.
First, CPU per-image vision-encoder time is essentially flat
($\sim$9.0--9.7\,s), indicating that the CPU path does not
exploit batch-level parallelism.
The Neutron NPU, by contrast, reduces Vis/img from 3198.6\,ms
at BS\,=\,1 to 2234.1\,ms at BS\,=\,8, a 30\% reduction, and
widens its speedup over CPU on vision encoding from
$2.82\times$ ($9029.8/3198.6$) at BS\,=\,1 to
$4.25\times$ ($9485.3/2234.1$) at BS\,=\,8.
Second, NPU throughput grows monotonically with batch size
through BS\,=\,8 (47.90\,tok/s), with no sign of saturation,
whereas CPU throughput exhibits a dip at BS\,=\,7
(28.87\,tok/s) before partially recovering at BS\,=\,8,
consistent with a memory-bandwidth bottleneck on the shared DDR
path.
Third, RSS grows by approximately 150\,MB per additional batch
sample on both back ends, so memory is the binding constraint
at large batch sizes on the i.MX95.
For the single-image inference setting reported in the main text
(BS\,=\,1), the NPU back end reduces TTFT from 10554.6\,ms to
4469.2\,ms (a $2.36\times$ speedup), bringing TTFT
within the range usable for interactive applications.